%% file: main.tex
\documentclass[11pt,letterpaper]{article}
\makeatletter
\def\input@path{{supporting/}}
\makeatother
\usepackage[margin=1in]{geometry}
\usepackage{natbib}
\setcitestyle{authoryear,round,citesep={;},aysep={,},yysep={;}}
\usepackage[T1]{fontenc}
\usepackage{amsmath,amssymb,booktabs,tikz,array}
\usepackage{amsthm}
\usepackage{graphicx}

\usepackage{listings}

\lstdefinestyle{lean}{
  basicstyle=\small\ttfamily,
  columns=fullflexible,
  keepspaces=true,
  frame=none,
  showstringspaces=false,
  literate=
    {α}{{$\alpha$}}1
    {ℕ}{{$\mathbb{N}$}}1
    {→}{{$\to$}}1
    {∀}{{$\forall$}}1
    {↔}{{$\leftrightarrow$}}1
    {∈}{{$\in$}}1
    {⊆}{{$\subseteq$}}1
    {≤}{{$\leq$}}1
    {∧}{{$\land$}}1
    {∉}{{$\notin$}}1
    {∃}{{$\exists$}}1
}

\newtheorem{result}{Theorem}
\newtheorem{suppthm}{Theorem}[subsection]
\newtheorem{supplem}[suppthm]{Lemma}
\newtheorem{suppcor}[suppthm]{Corollary}
\newtheorem{suppprop}[suppthm]{Proposition}
\theoremstyle{definition}

\theoremstyle{remark}

\theoremstyle{plain}
\usepackage[hidelinks]{hyperref}

\title{GenLimitLib: A Formal Library for Language Generation in the Limit and AI-Assisted Mathematical Research}
\author{Shuangping Li \\
  Yale University \\
  {\small\texttt{shuangping.li@yale.edu}}
  \and
  Peng Zhang \\
  Rutgers University \\
  {\small\texttt{pz149@cs.rutgers.edu}}}
\date{}
\begin{document}
\maketitle

\begin{abstract}
We present \textbf{GenLimitLib}, a source-aligned Lean 4 library for \emph{language generation in the limit}. Introduced by Kleinberg and Mullainathan at NeurIPS 2024, language generation in the limit studies a theoretical question motivated by LLMs: how to generate valid new strings from observed examples. This young and rapidly evolving field offers a natural testbed for studying large-scale formalization.
GenLimitLib contains formal developments for 30 papers. It extracts shared definitions and reusable proof components while preserving paper-specific assumptions and statements, and records relationships across papers. In this way, GenLimitLib provides a concrete and structured view of the literature. We show through mathematical case studies and LLM experiments how our library can support both human mathematical research and AI-assisted research.

\paragraph{Library:} \url{https://github.com/pengzhang91/generation-in-the-limit-lib}
\end{abstract}

\section{Introduction}

How should we study a young and rapidly evolving research area? Unlike mature fields with stable textbook treatments, its structure may still be unclear. For example: What are the common objects? Which assumptions are essential? Which proof ideas are reusable? How do results across papers fit together? Answering these questions is itself part of understanding the field.

We study \emph{Lean formalization} as one way to address this problem.
Formalizing a research literature forces definitions to be made precise and hidden assumptions to be stated explicitly.
During this process, we can ask whether similar concepts across papers are truly the same, which proof components can be reused, and whether results from different papers can be composed.
Common structures and reusable proof components can then be organized into shared layers and used in future research.
In this way, formalization provides a structural view of an evolving research literature.

\begin{figure}[tbp]
    \centering
    \includegraphics[width=0.55\textwidth]{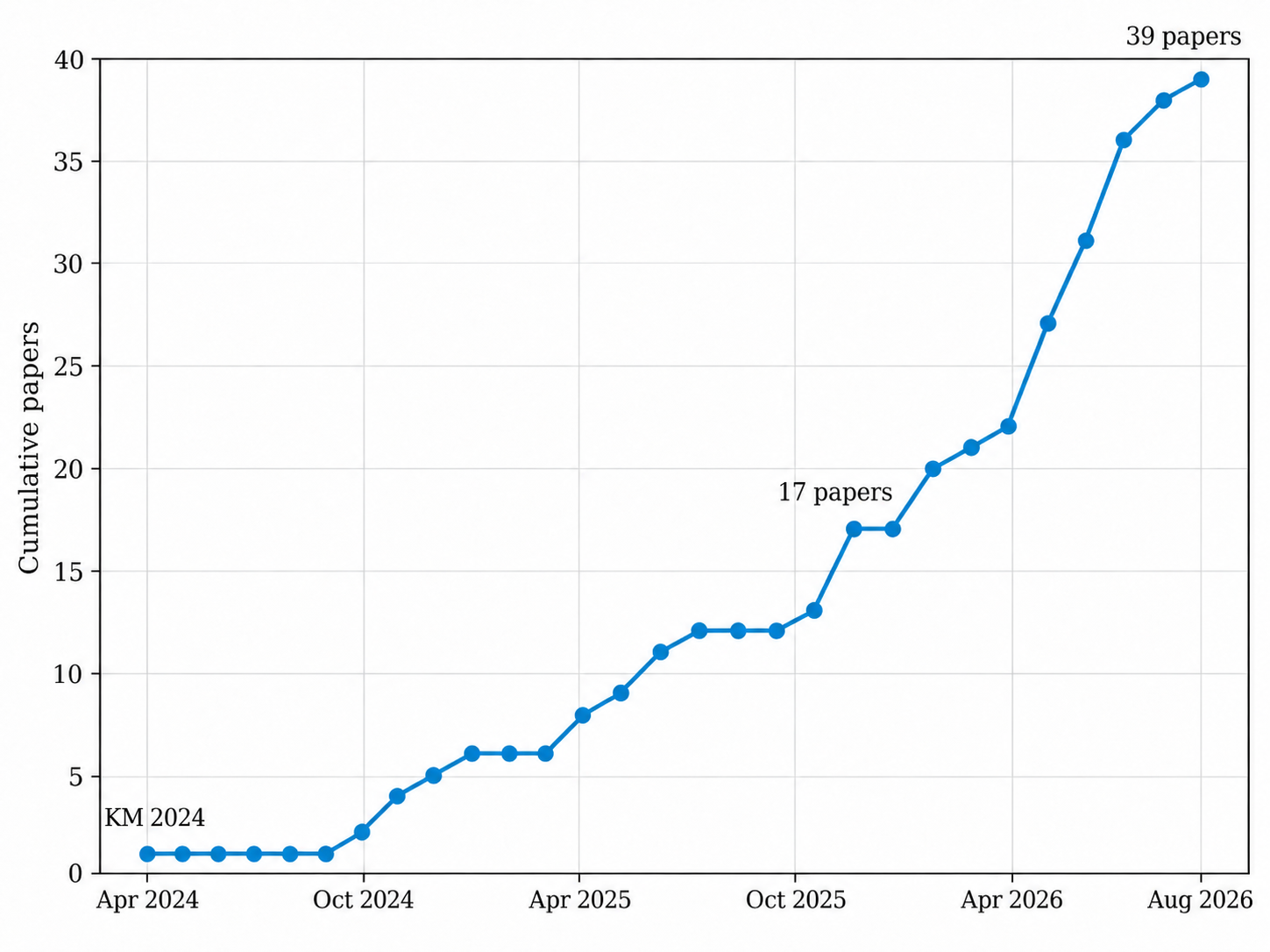}
    \caption{Cumulative number of papers in language generation in the limit (source: \url{languagegeneration.github.io}, accessed September 2026).}
    \label{fig:lgitl_growth}
\end{figure}
We explore this approach in \emph{language generation in the limit},
a theoretical area introduced by \citet[][NeurIPS]{km2024}.
It studies a fundamental mathematical question in language generation: how can a generator produce valid new examples from positive observations of an unknown language?
The literature has since grown rapidly, as shown in Figure~\ref{fig:lgitl_growth}.
It now includes work on language generation, identification, generation breadth, density, noise, feedback, and replay, as well as statistical, privacy, time, and memory \citep{li2024generation,kalavasis2025limits,kleinberg2025density,bai2026noise,replay,mehrotra2026differentially,ganju2026theory,memory}.
The area is large enough to contain substantial mathematical structure, yet small enough to formalize and connect a large fraction of the literature.

We build \emph{GenLimitLib}, a Lean 4 \citep{moura2021lean} library for language generation in the limit. Unlike libraries centered on broad domain foundations, GenLimitLib is source-aligned to a specific, evolving research literature. 
It extracts shared definitions and reusable proof components across papers while preserving paper-specific assumptions and statements.
It also records mathematical relationships across papers and theorems.
Therefore, GenLimitLib provides a concrete and \emph{structured} view of the research area.

This structured view can support human mathematical research.
In GenLimitLib, formalization exposes proof gaps, missing assumptions, and assumptions that could be weakened.
As more papers are formalized, extracting shared definitions and reusable proof
components also makes the global structure of the literature clearer.
This, in turn, can inspire new mathematical questions: whether two definitions
can be related, whether a construction can be transferred to another model, or
whether results from different papers can be combined into a new theorem.
Our case studies illustrate this interaction between formalization, paper reading, and human mathematical research.

The same structured view can also support AI-assisted mathematical research.
Instead of working only with individual papers, language models can use GenLimitLib to retrieve relevant mathematical claims and Lean declarations, reuse checked results, and develop formal proofs.
We study these capabilities by evaluating LLMs on mathematical understanding and on generating new Lean proofs for unpublished theorems with or without GenLimitLib.
More broadly, GenLimitLib may provide infrastructure for studying an \emph{iterative} AI-assisted mathematical research workflow: existing results are formalized and organized into reusable context, that context supports the development of new checked results, and those results can in turn become reusable infrastructure for future work. Our experiments study the key library-to-new-results step of this workflow.

We make the following contributions: 

\begingroup
\setlength{\leftmargini}{1.3em}
\begin{itemize}
    \item \textbf{A source-aligned formal library for an evolving research literature}. 
    We build GenLimitLib.
     It contains formal developments for 30 papers and 405 scoped claims, and over 124,000 lines of Lean code. It factors shared mathematical structure into 18 \texttt{Core} modules of common interfaces and 34 \texttt{Support} modules of reusable proof components, while preserving paper-specific assumptions and statements. 17 \texttt{Bridge} modules record explicit relationships across papers.

    \item \textbf{Mathematical findings from working across the library.}
    We identify and repair a gap in the proof of a published 
    theorem, without changing its algorithm and convergence bound.
    We transfer a published three-language bounded-memory construction to replay,
    reducing another published four-language impossibility example to three,
    and characterize exactly which finite families admit deterministic
    proper generation under replay. We also determine the exact tradeoff
    between mistakes and convergence for deterministic generators on a
    staircase family, resolving a special case of an open problem.

    \item \textbf{Experiments on AI use of a research library.}
    We evaluate GenLimitLib on both Lean proof development and mathematical reading, providing evidence that a structured research library can help LLMs develop new kernel-checked Lean proofs and better understand mathematical papers. Across five \emph{unpublished} theorem tasks and 300 independent proof-generation runs, access to the full library increases kernel-checked proof success from 48\% to 79\%. Separately, on mathematical-reading tasks, relevant Lean evidence and library-derived paper maps improve accuracy by 8–10 percentage points over their respective baselines.

\end{itemize}
\endgroup

\section{Related Work}

\paragraph{Formal libraries.} 
Existing formal libraries are often organized around broad mathematical foundations or particular domains. 
Mathlib provides broad reusable infrastructure for formal mathematics
\citep{mathlib2020}.
A growing set of domain-oriented formal libraries develops reusable
infrastructure for areas such as statistics, machine learning, computer science, including Statlib, Stat-Lean, StatsMLlib, AI4SLT, CSLib,
TCSlib, and EconCSLib
\citep{statlib2026,statlean2026,statmllib2026,zhang2026ai4slt,
cslib2026,tcslib2026,bei2026econcs,garg2026econcs}.
GenLimitLib instead focuses on a much smaller organizing unit: a specific, closely connected, and rapidly evolving research literature.

\paragraph{Formalization in AI-assisted mathematical research.}
Recent AI systems increasingly integrate formal methods into mathematical reasoning and research workflows.
AlphaProof performs formal proof search in Lean \citep{hubert2026alphaproof}, and Aristotle combines informal reasoning with Lean proof search and verification
\citep{achim2025aristotle}.
Google DeepMind's AI Co-Mathematician develops an agentic workbench for open-ended mathematical research and discusses formal provers as components for verified reasoning
\citep{zheng2026ai}.
At the research level, AxiomProver produced a Lean/Mathlib-verified proof of Fel's conjecture, and other systems combine mathematical exploration with formal proof search and verification on open problems
\citep{chen2026fel,ju2026conjecture,tsoukalas2026formal}.
MathCoPilot further combines retrieval, verification, and an accumulating formal knowledge base within a human-AI research workflow
\citep{zhang2026mathcopilot}.
LeanDojo and LeanSearch study retrieval and reuse from large formal libraries
\citep{yang2023leandojo,gao2026leansearch},
and Herald, LeanReasoner, and FANS use formal knowledge as context for mathematical reasoning
\citep{gao2025herald,jiang2024leanreasoner,yao2025fans}.

Our work studies a complementary role for formalization: rather than using it as a component of a particular research workflow, we study it as persistent research infrastructure.
GenLimitLib organizes a connected, evolving research literature in advance into a source-aligned formal library that can be repeatedly queried and reused for mathematical reading and Lean proof development.

\section{Building GenLimitLib}
\label{sec:library-construction}

We present GenLimitLib, a well-structured, source-aligned formal library for a young and rapidly evolving research literature. It extracts shared definitions, reusable proof components, and explicit cross-paper relationships, while preserving paper-specific assumptions and statements. 
This library supports systematic navigation, audit, maintenance, and reuse across the literature.
The library is designed to grow with the literature.

\subsection{Library Construction}
\label{sec:construction}

We began building GenLimitLib by formalizing Kleinberg and Mullainathan's foundational paper on language generation in the limit~\citep[][P01]{km2024}. The paper introduces the basic model and a simple generation algorithm (referred to as the KM algorithm): after observing sufficiently many positive examples from an unknown target language, the algorithm eventually produces unseen examples belonging to that language. Its semantic formalization\footnote{By semantic, we mean a formalization at an abstract level that may use noncomputable operations.} requires only lightweight Lean infrastructure, mostly on basic set-theoretic operations; on top of the shared \texttt{Core} definitions, the semantic development is only about 330 lines of Lean. We then formalized central parts of the classical identification theory of \citet[][P00]{gold1967language} and \citet[][P00A]{angluin1980}, including Gold's positive-text model and Angluin's finite-tell-tale characterization, which underlie many later connections between generation and identification.

\begin{figure}[t]
    \centering
    \includegraphics[width=0.8\linewidth]{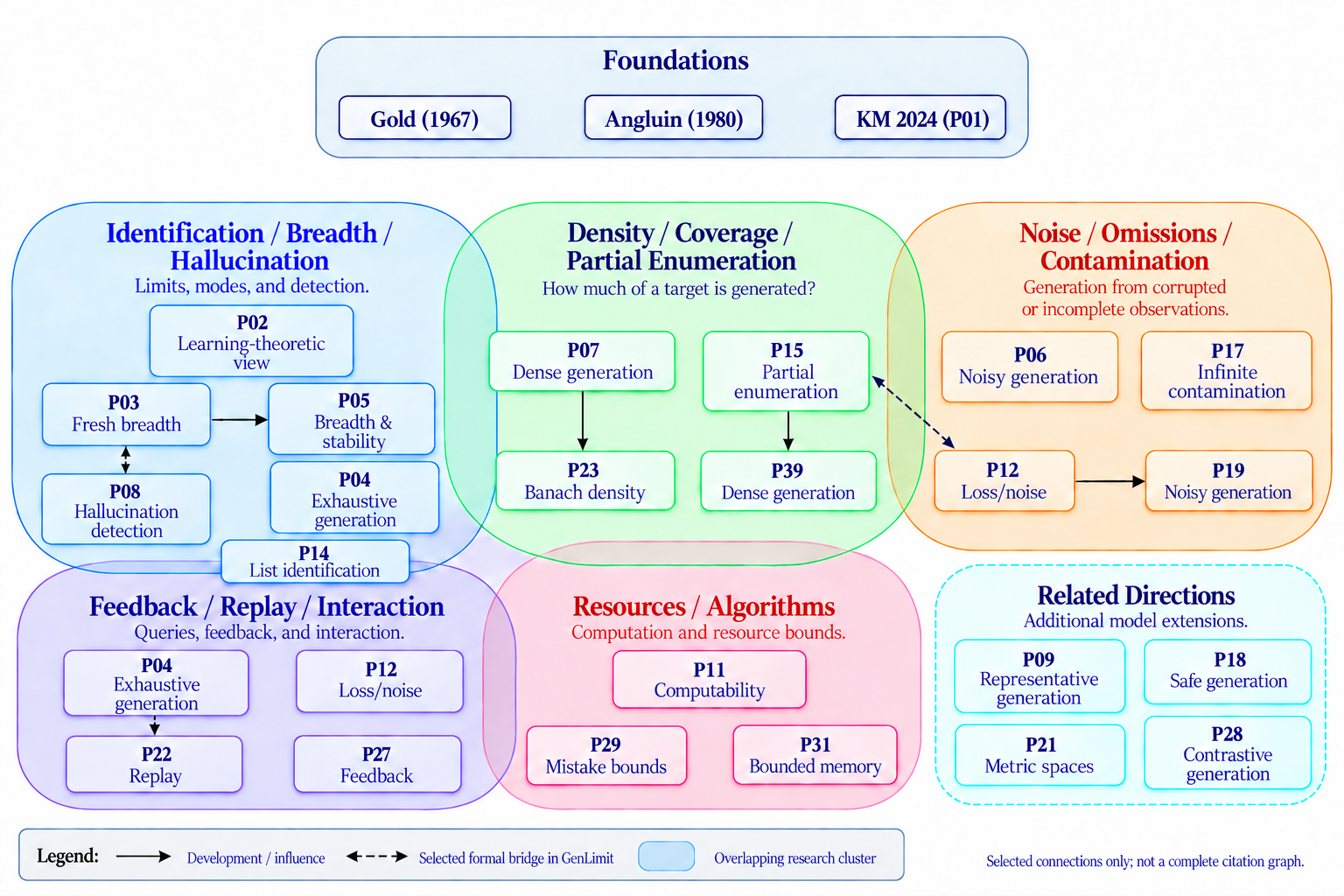}
    \caption{Research paper clusters in language generation in the limit. We follow the order in \url{languagegeneration.github.io} to number the papers.}
    \label{fig:paper_cluster}
\end{figure}

We used mathematical dependencies and emerging research clusters to guide the order and scope of formalization; see Figure~\ref{fig:paper_cluster}. Human researchers selected which papers and results to formalize, as well as the appropriate level of formalization.
Many results admit lightweight semantic formulations, but their finite-query, computability, probabilistic, or machine-level guarantees often require substantially richer Lean infrastructure that may not yet be available, as illustrated by CSLib \citep{cslib2026}.
For example, in GenLimitLib, P01 \citep{km2024} includes both semantic and finite-query formalizations; P03 \citep{kalavasis2025limits} focuses on the probability-free semantic cores; P30 \citep{ganju2026theory} currently formalizes deterministic components without the randomized almost-sure layer. 
In many cases, the semantic layer already captures a substantial part of the underlying mathematics.
Within each selected scope, we used AI-assisted Lean development, after which human researchers audited the resulting formalization.\footnote{Lean development used \texttt{codex/gpt-5.6-sol} at \texttt{ultra} and \texttt{xhigh} reasoning levels.}

As the library grew, formalization also revealed common structures and paper relationships that were not always obvious because of different notions and presentations across papers. 
We extracted shared definitions into \texttt{Core} and reusable proof components into \texttt{Support}. 
We kept each paper's assumptions and statements within its \texttt{Paper} modules, and recorded relationships across papers in cross-paper \texttt{Bridge} modules.

GenLimitLib currently represents 30 paper developments, totaling 124,275 lines of Lean code and 405 explicitly scoped source claims. Of these claims, 217 (53.6\%) have full coverage and 82 (20.2\%) have componentwise partial coverage, with some components formally proved in Lean.\footnote{Among the 82 claims with partial Lean coverage, one uncovered component might be a finite-query or executable realization, a randomized layer, runtime or sample-complexity bounds, typically outside the present formalization scope. Among the 106 claims with no Lean coverage, 63 are outside the current scope, 14 are inconsistent or contradicted as currently stated, and 5 have a source-proof gap or underspecification.}
The shared layer consists of 18 \texttt{Core}, 34 \texttt{Support}, and 17 \texttt{Bridge} modules. Of 252 \texttt{Core} declarations, 134 (53.2\%) are directly reused by at least two papers and 60 (23.8\%) by at least five; among 214 \texttt{Support} theorems, 79 (36.9\%) and 6 (2.8\%) are reused at these respective thresholds. Separately, the library records 34 explicit cross-paper relations spanning 19 directed paper pairs.

\subsection{Shared Mathematical Infrastructure}
\label{sec:shared_infrastructure_overview}

Our library has two shared layers.
\texttt{Core} contains paper-independent mathematical interfaces.
\texttt{Support} contains proof components reused across papers.
We use two representative examples to illustrate how the \texttt{Core} layer organizes shared mathematical structure across the literature. These examples also introduce the basic definitions of language and generation used in Section \ref{sec:our-results}.

\paragraph{Languages, families, and access models.}
A basic design principle of GenLimitLib is to separate the core mathematical objects
from model-specific assumptions and access mechanisms.
Therefore we keep definitions in \texttt{Core} general.
For example, the library separates three choices: the example space, the collection of languages, and the algorithmic access model.
In \texttt{Core}, we define
\begin{lstlisting}[style=lean]
abbrev Language (α : Type*) := Set α
abbrev LanguageClass (α : Type*) := Set (Language α)
abbrev LanguageFamily (α : Type*) := ℕ → Language α
\end{lstlisting}
A \texttt{Language} is simply a set and may be finite or infinite.
A \texttt{LanguageClass} is an arbitrary set of languages and may be
uncountable.
A \texttt{LanguageFamily} is instead indexed by $\mathbb{N}$, so its range is at most countable.
It also allows repeated languages.
The second representation is used by \citet[][P01]{km2024}, which considers a
countable indexed list of languages.

Additional assumptions are added only when required by a particular model.
For example, the KM finite-query formalization uses
\begin{lstlisting}[style=lean]
structure OracleFamily where
  language   : LanguageFamily
  infinite'  : ∀ i, (language i).Infinite
  query      : ℕ → ℕ → Bool
  query_spec : ∀ i u, query i u = true ↔ u ∈ language i
\end{lstlisting}
Here, \texttt{language} is the underlying semantic object.
The field \texttt{infinite'} adds the assumption that every language is
infinite.
The field \texttt{query} adds a uniform membership oracle.
These assumptions are conceptually separate from the language family itself.

The KM development shows why this separation is useful. 
The semantic construction needs only the indexed \texttt{LanguageFamily} and the infinitude of its languages. 
It reasons directly about relations between whole languages and is thus noncomputable. 
The finite-query construction uses the full \texttt{OracleFamily}, which includes the same underlying family and infinitude assumption, together with a membership oracle. 
It replaces whole-language tests by finite Boolean computations built from \texttt{query}. 
Thus, the semantic and finite-query developments share the same underlying language-family definitions, and stronger access assumptions such as membership queries are introduced only for the finite-query layer.

\paragraph{Generation guarantees.}
GenLimitLib also separates a common generation interface from different notions of eventual success.
A generator observes a finite history and outputs one new example.
For an example stream $w$, \texttt{sample w t} is the set of distinct observations
before time $t$.
\texttt{Presents w L} means that $w$ enumerates exactly the (target) language $L$. 
\texttt{StreamIn w L} only requires every observation to belong to $L$.
Finally, \texttt{CorrectAt G L w t} means that the output of $G$ at time $t$
belongs to $L$ and has not appeared in the observed sample from $w$.

With these common notions, the three main generation guarantees differ primarily in their quantifier structure: generation in the limit \citep{km2024}, and uniform and nonuniform generation \citep{li2024generation}. In \texttt{Core}, we define:
\begin{lstlisting}[style=lean]
def IsLimitGenerator
    (gen : Generator α) (H : LanguageClass α) : Prop :=
  ∀ L, L ∈ H → ∀ stream, Presents stream L →
    ∃ T, ∀ s, T ≤ s → CorrectAt gen L stream s

def IsNonuniformGenerator
    (gen : Generator α) (H : LanguageClass α) : Prop :=
  ∀ L, L ∈ H → ∃ d,
    ∀ stream, StreamIn stream L →
      ∀ t, (sample stream t).card = d →
        ∀ s, t ≤ s → CorrectAt gen L stream s

def IsUniformGeneratorAt
    (gen : Generator α) (H : LanguageClass α) (d : ℕ) : Prop :=
  ∀ L, L ∈ H → ∀ stream, StreamIn stream L →
    ∀ t, (sample stream t).card = d →
      ∀ s, t ≤ s → CorrectAt gen L stream s

\end{lstlisting}

The quantifier order captures the main differences.
For generation in the limit, the time threshold $T$ comes after both the target $L$ and its presentation, so it may depend on both.
Nonuniform generation chooses a sample-size threshold $d$ after the target $L$ but before the stream, so it may depend on the target but not on its presentation.
Uniform generation chooses the same threshold $d$ before the target is known, so the same $d$ must work for the whole class.
In all three notions, a single generator must work for the whole class.
For classes of infinite languages, \texttt{Core.ClassGeneration} further proves the hierarchy: uniform generation $\implies$ nonuniform generation $\implies$ generation in the limit.

\texttt{Core} contains 18 top-level Lean modules. Beyond the two examples above, these modules provide shared interfaces for language identification, learning-theoretic structures, density and partial presentations, noise and robustness, and other topics.
\texttt{Support} contains an additional 34 Lean modules. These modules collect reusable proof and construction components shared across different papers, including fresh-element selection, finite-tell-tale, adaptive membership-query trees, history-chain constructions, and stabilization lemmas.

\subsection{Library Navigation and Maintenance}
\label{sec:source_access}

To support both use and continued growth, GenLimitLib provides navigation at three levels: papers, mathematical claims, and Lean declarations.
These levels serve complementary uses: PaperMaps are human-readable guides, claim cards support claim-level retrieval for humans and language models, and Lean declaration cards support Lean-level retrieval for proof development.

\paragraph{Paper-level navigation.}
For each paper, a \texttt{PaperMap} provides a compact human-readable guide in Markdown.
It records the paper source version, formalization boundary, main theorem entry points, and known gaps or repairs.
It also points to shared infrastructure and relevant cross-paper relationships.

\paragraph{Claim-level navigation.}
For finer-grained retrieval, we provide claim cards.
Each paper has a machine-readable JSON file in \texttt{registry/papers/} containing its claim records.
These claims are also exported as independent entries in \texttt{cards.jsonl}.
Each card gives a short description of the claim, links it to relevant Lean declarations, and records the scope and progress of the formalization.
The cards provide finer retrieval units than the original papers and support claim-level search by both humans and language models.

\paragraph{Lean declaration-level navigation.}
For navigation from the Lean side, we provide Lean declaration cards.
Each public Lean declaration has a machine-readable entry in
\texttt{registry/generated/declarations.jsonl}, together with a lookup index in \texttt{declaration-index.json}.
Each card records the declaration name, type, defining module, documentation, and direct dependencies within GenLimitLib.
When available, it also links back to the corresponding claim cards.
These cards support retrieval of definitions and theorems for Lean proof development.

\paragraph{Audit and maintenance.}
We use three cumulative levels of human audit. Level 1 checks theorem statements: whether the assumptions, inputs, outputs, and conclusions in Lean match the paper. Level 2 additionally checks algorithms: whether the formal construction or state machine matches the paper's algorithm. Level 3 further checks intermediate lemmas and proof dependencies against their source counterparts.
Each audit tracks the Lean declarations relevant to its scope.
We store fingerprints of their theorem statements and definitions.
When one of these declarations changes, only the affected audit scopes are marked for review, and unrelated audits remain valid.
This allows a narrow re-audit of the changed part rather than re-auditing the whole paper.

\section{Theoretical Findings}
\label{sec:our-results}
\label{sec:research-method}
We present three mathematical case studies developed while building and working with GenLimitLib: a proof gap with a local repair, a construction transferred between papers, and a solution to an open problem for the staircase family. All mathematical results below have been formalized and verified in Lean. Appendix~\ref{app:proof-status} gives the research workflow, written proofs, and formal verification. 

\input{supporting/results_connections.tex}
\input{supporting/results_math.tex}

\section{Experiments}
\label{sec:experiments}
\label{sec:method}

We empirically evaluate GenLimitLib on Lean proof formalization and mathematical reading.

\input{supporting/results_autoformalization.tex}

\input{supporting/e01s_brief.tex}
\section{Discussion and Conclusion}

We build GenLimitLib, a formal library for a young and rapidly evolving research area. It extracts shared definitions, reusable proof components, and records explicit cross-paper relationships. We provide evidence that this structure can support both human and AI-assisted mathematical research. Below, we discuss several future directions.

\paragraph{Using formal research libraries.}
Our PML-Oracle results suggest that retrieval can be a separate bottleneck. Better theorem search and dependency-aware retrieval are natural next steps. The library may also provide useful data for domain-specific post-training.

\paragraph{Automating library construction more.}
Building a coherent multi-paper library requires human experts to make choices about scope, shared abstractions, reusable proofs, and cross-paper relationships. Future systems could automate some of this process.

\paragraph{Beyond language generation in the limit.}
Our approach is not specific to language generation in the limit. Other young research areas with closely related papers, shared definitions, and reusable proof ideas may also benefit from literature-centered formal libraries.

\subsection*{Reproducibility statement}
We provide reproducibility materials for both the theoretical and experimental results at \url{https://github.com/pengzhang91/genlimitlib-experiments}.
Appendix~\ref{app:proof-status} gives complete written proofs for the main mathematical findings and records the corresponding scope of Lean verification.
Appendix~\ref{app:autoformalization} specifies the proof-generation setup and additional experimental details.
Appendix~\ref{app:experiment-details} provides the mathematical-reading protocol, question construction, reader configuration, and analysis procedure.

\subsection*{AI Use Statement}

Generative AI tools (GPT-5.6 and 6, Claude 5) were used to assist with literature review; mathematical exploration, question and theorem formulation, and statement refinement; Lean formalization and proof development; experiment implementation and analysis; and manuscript drafting and editing. The mathematical results reported in the paper were reviewed by the authors and independently verified in Lean where stated. Experimental code and reported results were checked by the authors. The authors take responsibility for the final content of the paper, including all AI-assisted code, mathematical claims, and text.

\subsection*{Acknowledgements}

The authors would like to thank Ziyi Cai, Yiheng Shen, and Kangning Wang for helpful discussions at an early stage of this work. We also thank Moses Charikar and Chirag Pabbaraju for insightful early discussions and for raising several helpful questions.

\begingroup
\raggedright
\bibliographystyle{supporting/conference}
\bibliography{supporting/references}
\endgroup
\appendix
\raggedbottom
\input{supporting/appendix_roadmap.tex}
\input{supporting/appendix_proof_guide.tex}
\input{supporting/appendix_replay.tex}
\input{supporting/appendix_p29_frontier.tex}

\input{supporting/appendix_autoformalization.tex}

\input{supporting/appendix_experiments.tex}

\end{document}

%% file: supporting/results_connections.tex
\subsection{Identifying and repairing a proof gap}
\label{sec:precision}
In our first example we identify a gap in the proof of a published generation theorem and give a short repair. The repair preserves the original algorithm and theorem statement.

In \emph{Pareto-optimal Non-uniform Language Generation},
\citet[][P13]{charikar2026pareto} construct a generator whose convergence
guarantee depends on the target language but holds for every presentation
order. Their proof of Theorem~8 uses a procedure that inserts languages
into an ordered list. When a language $L$ is inserted, the procedure
considers subcollections containing $L$, drawn from the list up to $L$,
whose intersection is finite. It stores the largest such intersection
size as a score $m^\star(L)$, together with a subcollection
$\mathcal C(L)$ attaining it. If no such subcollection exists, it stores
zero and an empty subcollection. Claim~7 asserts that the stored
subcollection remains a maximizer as further languages are inserted.

While formalizing Claim~7, GenLimitLib records a counterexample. Take two disjoint infinite
languages $A$ and $B$. When $A$ is inserted first, no subcollection
containing $A$ has finite intersection, so the procedure stores
$\mathcal C(A)=\varnothing$ and $m^\star(A)=0$. When $B$ is inserted,
$\{A,B\}$ has intersection size zero, and the insertion rule places
$B$ before $A$. However, the stored subcollection $\mathcal C(A)$
remains empty. It does not contain $A$, so it cannot be the
maximizing subcollection asserted by Claim~7.

Our repair proves that every subcollection containing $L$, drawn from the current list up to $L$ and having finite intersection, has
intersection size at most the stored score $m^\star(L)$.
It proves that this bound is preserved after each insertion.
This is the bound needed to complete the generation proof.

\subsection{Connecting papers to derive new results}
\label{sec:connections}

GenLimitLib also supports transferring constructions between different learning
problems. We connect a bounded-memory example from P31 to the replay
model of P22, obtaining a stronger impossibility example. This connection
then motivates an exact characterization result.

\emph{Language Generation with Replay} \citep[][P22]{replay}
studies a model in which the adversary can use the generator's earlier outputs to contaminate future observations. The paper studies several generation guarantees in this setting. We focus on \emph{proper generation}: proposals must come from a fixed family and eventually be contained in the true target.

A different paper, \emph{On Language Generation in the Limit with
Bounded Memory} \citep[][P31]{memory}, asks how restricting a
generator's memory affects generation. It gives a three-language example on which generators remembering only their previous proposed
index cannot succeed. The languages are
$T\cup\{a,b\}$, $T\cup\{a,c\}$, and $T\cup\{b,c\}$,
where $T$ is countably infinite and $a,b,c$ are distinct examples outside $T$.

Using GenLimitLib, we show that this same example also makes deterministic proper
generation under replay impossible, even with unrestricted memory.
Our proof combines P31's construction with the indistinguishability
argument in P22's impossibility proof. This improves
P22's four-language impossibility example to three languages.

This connection motivates a broader question: which finite
families permit proper generation under replay?
Our Theorem~\ref{thm:proper-replay} below gives an exact characterization.
\begin{result}[When proper generation under replay is possible]
\label{thm:proper-replay}
Let $L_1,\ldots,L_N$, with $N\ge1$, be a finite family of infinite countable
languages. Deterministic proper generation under replay is possible
if and only if, for every possible first observation $x$, there is
a first proposal $L_j$ with the following property.
Group the indices $i$ with $x\in L_i$ according to equality of
$L_i\setminus L_j$. For every such group $E$, some language $L_h$
from the original family satisfies
$L_h\subseteq\bigcap_{i\in E}L_i$.
\end{result}

%% file: supporting/results_math.tex
\subsection{Resolving an open problem in a special case}
\label{sec:new-mathematics}
\emph{Mistake-Bounded Language Generation} \citep[][P29]{mistake}
studies generators that learn from positive examples and output an unobserved
example at each step. A mistake occurs when the output is outside the target
language. The paper asks how the total number of mistakes trades off against
how quickly the generator stops making them (Section~8, Open Direction~(2)). Guided by GenLimitLib's definitions of mistake and convergence
guarantees, we resolve this question for deterministic generation on a
\emph{staircase family} that generalizes the paper's tradeoff construction.

This \emph{staircase family} consists of languages $L_i=P_i\cup R_i$,
where $P_i=\bigcup_{k\le i}C_k$. The blocks $C_i$ are nonempty and finite,
the sets $R_i$ are countably infinite, and all these sets are mutually disjoint.
Thus each successive language contains one more finite block, together with
an infinite part $R_i$ that belongs only to that language. Write $s_i=|P_i|$.
We measure convergence by the number of distinct observations needed before
all later outputs are guaranteed to be valid, for every presentation order.

For this family, the tradeoff comes from a simple choice. When the observed
set is exactly $P_k$, the generator can output in $R_k$, which is valid only
for $L_k$, or in $C_{k+1}$, which is valid for every later target.
Away from these decision points, it can choose an unobserved example common to all
compatible targets. Let $\sigma_k=1$ record the choice of $R_k$ and
$\sigma_k=0$ the choice of $C_{k+1}$. Write $M_i(G)$ for the worst-case
number of mistakes of a generator $G$ on $L_i$, and $D_i(G)$ for its
convergence deadline.

\begin{result}[Optimal mistake and convergence guarantees]
\label{thm:staircase-frontier}
For a binary sequence $\sigma=(\sigma_k)_{k\ge1}$, let
$J_i(\sigma)=\{k<i:\sigma_k=1\}\cup\{i:\sigma_i=0\}$ and define
\[
 m_i(\sigma)=|J_i(\sigma)|,\qquad
 d_i(\sigma)=\max\bigl(\{0\}\cup\{s_k+1:k\in J_i(\sigma)\}\bigr).
\]
Each sequence $\sigma$ is realized by a deterministic generator $G_\sigma$
with $M_i(G_\sigma)=m_i(\sigma)$ and $D_i(G_\sigma)=d_i(\sigma)$ for all $i$.
Conversely, every deterministic generator $G$ admits a sequence $\sigma$
such that $m_i(\sigma)\le M_i(G)$ and $d_i(\sigma)\le D_i(G)$ for every $i$.
These are exactly the Pareto-optimal guarantees: none can be improved
without worsening another.
\end{result}

The theorem gives matching constructions and lower bounds for every optimal
tradeoff on this family. It determines both the number and the latest possible
timing of mistakes, for arbitrary positive finite block sizes
(Appendix~\ref{app:p29-frontier}).

%% file: supporting/results_autoformalization.tex
\subsection{Library-Assisted Lean Proof Generation}
\label{sec:autoformalization}

We use a five-paper research-neighborhood sub-library centered on paper 
\citet[][P39]{cai2026dense} studying dense generation.\footnote{\citet[][P39]{cai2026dense} is close to the foundational algorithm of \citet[][P01]{km2024} and simplifies the dense-generation algorithms of \citet[][P07]{kleinberg2025density} and \citet[][P15]{kleinberg2025partial}.}
The four other papers study variants of the basic model involving noise, omission, feedback, and contamination \citep{ananth2025generation,bai2026noise,mehrotra2026language,li2026characterizing}.
These five papers form a small but closely connected body of work, providing a realistic setting for studying research within a focused cluster of related papers.

We use five theorem tasks obtained from 54 candidate proposals generated by an upstream LLM ideation pipeline over this research neighborhood and the sub-library; see Appendix \ref{app:autoformalization_more_details}.
The pipeline targeted new results within this neighborhood and screened out proposals that were false or immediate consequences of the five papers.
The retained statements were then independently screened and formally validated. They were \emph{genuinely new and unpublished} at the time of evaluation.

\begin{table}[htbp]
\centering
\small
\caption{Resource conditions.}
\label{tab:resource-conditions}
\smallskip
\renewcommand{\arraystretch}{1.15}
\begin{tabular}{@{}ll@{}}
\toprule
Condition & Resources \\
\midrule
P
& Minimal Lean vocabulary \\
PML
& Full research sub-library \\
PML-Oracle
& Oracle-selected library subset \\
\bottomrule
\end{tabular}
\end{table}
We compare three resource conditions. All conditions received the same five papers, faithful Lean theorem statement, Mathlib environment, and author instructions; they differed only in the project-specific resources provided to the LLM author, in Table~\ref{tab:resource-conditions}.
P contains only a minimal definition-level project vocabulary. PML additionally provides the full frozen sub-library with 161 Lean files and navigation materials.
PML-Oracle replaces the 161 Lean files with a task-specific subset of Lean files selected using an oracle reference proof; the proof itself and API use are not exposed to PML-Oracle. 
We treat PML-Oracle as an idealized file-retrieval condition.
We test these conditions with or without a given natural-language proof and at high or medium \texttt{gpt-5.6-sol} reasoning effort.

For each of the five theorem tasks, we run five independent trajectories under each combination of the three resource conditions, with or without a supplied proof, and at high or medium reasoning effort, yielding 300 runs in total.
All runs used Codex CLI 0.155.1, a 90-minute LLM author budget, Lean 4.24.0, and a locked Mathlib/Lake environment.

Success is binary. A run is successful if the LLM-produced Lean proof independently recompiles and passes kernel checking against the frozen theorem statement, which the model is not allowed to modify. We allow only \texttt{propext}, \texttt{Classical.choice}, and \texttt{Quot.sound}.

\paragraph{Experiment results.}
Table~\ref{tab:stage3-results}(a) shows the main factorial results.
The aggregated success is associated with access to the research library.
Across all proof and reasoning conditions, PML succeeds on 79/100 runs, compared with 48/100 for P; PML-Oracle succeeds on 81/100.
Importantly, the PML advantage remains when no natural-language proof is supplied: PML succeeds on 40/50 runs compared with 23/50 for P.
With a supplied proof, the corresponding numbers are 39/50 and 25/50.

Table~\ref{tab:stage3-results}(b) shows that tasks with heavier dependence on existing library benefit more from library access.
Task B is solved reliably even with P, while Task D and E are never solved by P in any of their 20 trajectories.
PML solves all 20 Task D trajectories and 5/20 Task E trajectories.
The corresponding successful PML proofs have library dependency shares of 91.9\% and 87.3\%, respectively, showing substantial reuse of the formal research library.
PML-Oracle further increases Task E success from 5/20 to 9/20.
For Task E, all PML and PML-Oracle successes occur at high reasoning effort; both remain at 0/10 under medium reasoning. This suggests that retrieval and reasoning effort can be separate bottlenecks.

\begin{table}[t]
  \centering
  \footnotesize
  \caption{
    Results across resource conditions.
    Left: aggregate success by proof availability and reasoning effort.
    Right: task-level success (H/M denotes high/medium reasoning, each out of
    ten trajectories) and proof-relevant library reuse; library share is
    computed over successful proofs.
  }
  \label{tab:stage3-results}

  \begin{minipage}[t]{0.39\columnwidth}
    \vspace{0pt}
    \centering
    \textbf{(a) Aggregate success}\\[4pt]

    \setlength{\tabcolsep}{2.5pt}
    \renewcommand{\arraystretch}{1.08}
    \begin{tabular}{lcccc}
      \toprule
      & \multicolumn{4}{c}{Success} \\
      \cmidrule(lr){2-5}
      Condition
          & \vphantom{\shortstack{Success\\H/M}} P
          & PML
          & \shortstack[c]{PML-\\Oracle}
          & Overall \\
      \midrule
      Proof / H
        & 15/25 & 22/25 & 25/25 & 62/75 \\
      No proof / H
        & 14/25 & 22/25 & 24/25 & 60/75 \\
      Proof / M
        & 10/25 & 17/25 & 16/25 & 43/75 \\
      No proof / M
        &  9/25 & 18/25 & 16/25 & 43/75 \\
      Overall
        & 48/100 & 79/100 & 81/100 & 208/300 \\
      \bottomrule
    \end{tabular}
  \end{minipage}
  \hspace{1.75em}
  \begin{minipage}[t]{0.56\columnwidth}
    \vspace{0pt}
    \centering
    \textbf{(b) Task-level success and reuse}\\[4pt]

    \setlength{\tabcolsep}{2.2pt}
    \renewcommand{\arraystretch}{1.08}
    \begin{tabular}{cccccc}
      \toprule
      & P
      & \multicolumn{2}{c}{PML}
      & \multicolumn{2}{c}{PML-Oracle} \\
      \cmidrule(lr){2-2}
      \cmidrule(lr){3-4}
      \cmidrule(lr){5-6}
      Task
        & \shortstack{Success\\H/M}
        & \shortstack{Success\\H/M}
        & \shortstack{Lib.\\share}
        & \shortstack{Success\\H/M}
        & \shortstack{Lib.\\share} \\
      \midrule
      A
        &  9/4 &  9/6  & 16.8\% & 10/5  & 16.1\% \\
      B
        & 10/9 & 10/9  & 57.6\% & 10/9  & 41.6\% \\
      C
        & 10/6 & 10/10 & 52.7\% & 10/8  & 50.7\% \\
      D
        &  0/0 & 10/10 & 91.9\% & 10/10 & 93.3\% \\
      E
        &  0/0 &  5/0  & 87.3\% &  9/0  & 87.4\% \\
      \bottomrule
    \end{tabular}
  \end{minipage}
\end{table}

\begin{table}[htbp]
\centering
\small
\caption{API efficiency.}
\label{tab:stage3-api-efficiency}
\smallskip
\renewcommand{\arraystretch}{1.15}
\setlength{\tabcolsep}{6pt}
\begin{tabular}{@{}lcc@{}}
\toprule
Condition & Valid proofs & \shortstack{API cost /\\ valid proof} \\
\midrule
P          & 48/100 & \$2.76 \\
PML        & 79/100 & \$1.95 \\
PML-Oracle & 81/100 & \$1.78 \\
\bottomrule
\end{tabular}
\end{table}

Table~\ref{tab:stage3-api-efficiency} reports aggregate API cost.
Compared with P, PML has a higher success rate and a lower API cost per valid proof; PML-Oracle performs even slightly better.
The API cost per valid proof is measured as the total API cost over all scheduled 100 runs divided by the number of kernel-validated proofs.

%% file: supporting/e01s_brief.tex
\subsection{Mathematical reading with library-derived evidence}
\label{sec:library-value}
\label{sec:reading-method}

After a closed-book screening, we evaluate Qwen3.6-27B on 404 questions about
individual papers and 204 questions on five closely related papers
on language generation and noise, chosen for their dense connections through
shared definitions and results. 
The source papers are listed in Appendix~\ref{app:experiment-details}.
For individual papers, we supplement excerpts
with relevant Lean definitions and theorem statements selected using the
questions' recorded sources. For the five-paper collection, we test whether a
paper map summarizing theorems and their relationships from library imports
and dependencies helps the model read. We compare the same excerpts with no
map, the relevant map, or a \emph{sham map}: Lean statements from outside the
collection, matched to the relevant map's heading and character length.
For details, see Appendix~\ref{app:experiment-details}.

\begin{table}[htbp]
\centering
\caption{Lean evidence: 404 questions.}
\label{tab:exp-lean}
\smallskip
\small
\renewcommand{\arraystretch}{1.15}
\setlength{\tabcolsep}{6pt}
\begin{tabular}{@{}lr@{}}
\toprule
Information supplied & Accuracy (\%) \\
\midrule
Question only & 41.20 \\
Paper excerpts & 70.09 \\
Paper + unrelated Lean & 71.12 \\
Paper + Lean in file order & 73.45 \\
\textbf{Paper + relevant Lean} & \textbf{78.55} \\
\bottomrule
\end{tabular}
\end{table}

\begin{table}[htbp]
\centering
\caption{Paper map: 204 questions.}
\label{tab:exp-map}
\smallskip
\small
\renewcommand{\arraystretch}{1.15}
\setlength{\tabcolsep}{6pt}
\begin{tabular}{@{}lr@{}}
\toprule
Map supplied & Accuracy (\%) \\
\midrule
No map & 60.26 \\
\textbf{Relevant map} & \textbf{70.61} \\
Sham map & 60.51 \\
\bottomrule
\end{tabular}
\end{table}
Relevant Lean raises accuracy from 70.09\% with excerpts alone to
\textbf{78.55\%} (Table~\ref{tab:exp-lean}); excluding 15 questions with
out-of-scope or study-derived evidence leaves accuracy at 78.40\%.
The paper map raises accuracy from 60.26\% to \textbf{70.61\%}, compared
with 60.51\% for the sham map (Table~\ref{tab:exp-map}), with gains mainly
outside questions requiring two papers. Together, these results suggest that
curated library content helps mathematical reading. They measure the benefit
of the selected information; an equally informative natural-language comparison
would clarify the contribution of formal representation.

%% file: supporting/appendix_roadmap.tex
\section*{Appendix Roadmap}
The three appendices follow the order of the main-text studies:
\begin{description}
\item[Appendix~\ref{app:proof-status}: Theoretical findings (Section~\ref{sec:our-results}).]
The research design and workflow come first in Appendix~\ref{app:research-method}.
Complete written proofs and their Lean verification status then follow the three
case studies in Sections~\ref{sec:precision}--\ref{sec:new-mathematics}.
\item[Appendix~\ref{app:autoformalization}: 
Lean proof generation (Section~\ref{sec:autoformalization}).]
Additional experiment details on theorem-task generation and selection, author-visible resources and prompts, and additional experiment results.
\item[Appendix~\ref{app:experimental-documentation}: Mathematical reading (Section~\ref{sec:library-value}).]
Source papers, reader configuration, and accuracy calculation accompany
the controlled reading experiments.
\end{description}

%% file: supporting/appendix_proof_guide.tex
\section{Theoretical Findings: Research Workflow and Proofs}
\label{app:proof-status}

This appendix first describes the research design and workflow used to
develop and check the findings in Section~\ref{sec:our-results}.
Complete written proofs follow in the order of the three main-text cases.

\input{supporting/appendix_research_workflow.tex}

\paragraph{Overview of proofs and verification.}
Table~\ref{tab:proof-guide} locates each proof and summarizes its
Lean verification.

\begin{table}[ht]
\centering
\caption{Proofs of the featured findings and scope of formal verification.}
\label{tab:proof-guide}
\small
\begin{tabular}{@{}>{\raggedright\arraybackslash}p{0.30\linewidth}>{\raggedright\arraybackslash}p{0.15\linewidth}>{\raggedright\arraybackslash}p{0.47\linewidth}@{}}
\toprule
Finding & Proof & Lean scope \\
\midrule
Section~\ref{sec:precision}: local invariant repair & Appendix~\ref{app:pareto-repair} & Counterexample, corrected score bound, and the repaired generation guarantee, including repeated observations. \\
Theorem~\ref{thm:proper-replay}: finite proper replay & Appendix~\ref{app:replay} & Three-language impossibility, both directions of the characterization, and finite-state consequences. \\
Theorem~\ref{thm:staircase-frontier}: exact staircase frontier & Appendix~\ref{app:p29-frontier} & Exact guarantees, lower bounds for arbitrary deterministic generators, and Pareto optimality. \\
\bottomrule
\end{tabular}
\end{table}

\paragraph{Proof materials.}
All mathematical results proved in this appendix have been formalized and
verified in Lean. The complete Lean proofs and instructions for reproducing
the verification are available at \url{https://github.com/pengzhang91/genlimitlib-experiments}.

\subsection{Identifying and repairing a proof gap in Pareto-optimal generation}
\label{app:pareto-repair}

In this subsection we give a repair to the proof of the generation-time
theorem of \citet[][P13]{charikar2026pareto}. During formalization we
found that published Claim~7 (arXiv Claim~3.2), used in published
Theorem~8 (arXiv Theorem~4), can fail when Procedure~1 stores an empty
subcollection to indicate that no candidate exists. We explain the
claim's role, exhibit the failure, and prove the numerical bound needed
to complete the generation argument. The repaired proof preserves the
original algorithm and its generation-time guarantee. References below
use the published ALT~2026 numbering.

\subsubsection{The insertion procedure and the original claim}

Let $L_1,L_2,\ldots$ be infinite languages over a countable universe $U$.
Procedure~1 inserts these languages one at a time into a finite ordered
list. It assigns each inserted language $L_i$ a nonnegative integer
$m_i$ and a stored subcollection $C_i$; both remain fixed after that
insertion. Here $m_i$ denotes the source paper's $m^*(L_i)$. The order,
and hence the languages preceding $L_i$, can change when later languages
are inserted.

Identify a subcollection with its set of indices. For a finite
index set $S$, write
\[
 I(D)=\bigcap_{j\in D}L_j,
 \qquad
 \mathcal C_i(S)=
 \{D\subseteq S:i\in D\text{ and }I(D)\text{ is finite}\}.
\]
Thus a member of $\mathcal C_i(S)$ is a subcollection containing $L_i$
whose intersection is finite. Define
\[
 M_i(S)=\max\bigl(\{0\}\cup
       \{|I(D)|:D\in\mathcal C_i(S)\}\bigr).
\]
The maximum exists because $S$ has finitely many subcollections. If
$\mathcal C_i(S)$ is nonempty, it is the largest finite intersection
size attained by a candidate. If the candidate set is empty, its value
is zero by convention. In particular, $M_i(S)=0$ does not distinguish
an empty candidate set from a nonempty candidate set whose intersections
are all empty.

Suppose a new language $L_q$ is immediately to the right of $L_i$ during
its insertion, and let $S_i$ consist of $i$ and all indices preceding
$i$ in the old order. The procedure computes
\[
 m_{\mathrm{check}}=M_q(S_i\cup\{q\})
\]
and moves $q$ to the left of $i$ if
$m_{\mathrm{check}}\le m_i$. It repeats this comparison while moving
leftward. The checked maximum is recomputed in the relevant prefix;
it need not equal the score eventually assigned to $q$. Once the
insertion ends, let $S_q$ be the prefix through $q$. The procedure sets
$m_q=M_q(S_q)$ and chooses $C_q\in\mathcal C_q(S_q)$ attaining this
maximum if a candidate exists. Otherwise it sets $C_q=\varnothing$
and $m_q=0$.

Claim~7 asserts that the stored subcollection for an old language
continues to be a maximizing candidate in its current prefix. In the
notation above, its maximizing-witness assertion is
\[
 C_i\in\mathop{\rm arg\,max}_{D\in\mathcal C_i(S_i)} |I(D)|.
\]
The proof of Theorem~8 uses this assertion to bound the size of a finite
intersection involving $L_i$. Specifically, the generator scans the
ordered languages, retaining a language when it contains every observed
point and its addition leaves the retained languages with an infinite
intersection. To show that the true target $L_i$ is retained after more
than $m_i$ distinct observations, the proof needs the implication
\[
 D\in\mathcal C_i(S_i)\quad\Longrightarrow\quad |I(D)|\le m_i.
\]
The assertion that the particular stored subcollection $C_i$ attains
the maximum is stronger than this implication.

\subsubsection{Failure of the maximizing-witness assertion}

Let $A$ and $B$ be disjoint infinite languages, inserted in that order.
At the first insertion, the only subcollection containing $A$ is
$\{A\}$, whose intersection is infinite. There is therefore no
candidate, and the procedure stores $C_A=\varnothing$ and $m_A=0$.

When $B$ is inserted, the subcollection $\{A,B\}$ has empty
intersection. It is a candidate for $B$ with score zero, whereas
$\{B\}$ still has infinite intersection. Consequently the checked
maximum is zero. The comparison $0\le m_A=0$ moves $B$ before $A$,
giving the order $(B,A)$. The prefix through $A$ now contains both
languages, and $\{A,B\}$ is a candidate for $A$ attaining the
maximum zero. However, the stored subcollection for $A$ is still
$\varnothing$. It does not contain $A$, so it is not a candidate and
cannot belong to the displayed argmax.

The failure is the identification of an empty stored subcollection
with a maximizing candidate. The empty subcollection contains no
languages; the candidate $\{A,B\}$ contains two languages whose
intersection contains no points. Assigning score zero in both cases
does not make the two subcollections interchangeable. The numerical
inequality needed by Theorem~8 remains true in this execution: every
candidate in the prefix through $A$ has intersection size at most
$m_A=0$.
The published induction controls the maximum value, but this alone
does not ensure that the stored subcollection is still a candidate.

\subsubsection{The corrected statement and its preservation}

The repair replaces the assertion about a stored maximizing
subcollection by an upper bound on every candidate's intersection
size. The following lemma provides the induction step: it shows
that the numerical bound persists when a new language enters an
old language's prefix.

\begin{supplem}[Preservation of the score bound]
\label{lem:pareto-score-invariant}
Let $S$ be a finite set of indices and let $b\in\mathbb N$.
Suppose that every $D\in\mathcal C_i(S)$ satisfies $|I(D)|\le b$.
If $q\notin S$ and $M_q(S\cup\{q\})\le b$, then every
$D\in\mathcal C_i(S\cup\{q\})$ also satisfies $|I(D)|\le b$.
\end{supplem}
\begin{proof}
Fix $D\in\mathcal C_i(S\cup\{q\})$. There are two cases.
If $q\notin D$, then $D\subseteq S$, and it still contains $i$
and has finite intersection. Thus $D\in\mathcal C_i(S)$, so the
assumed bound gives $|I(D)|\le b$.

If $q\in D$, then $D\subseteq S\cup\{q\}$ contains $q$ and has
finite intersection. It is therefore also a candidate in the
optimization defining $M_q(S\cup\{q\})$. Hence
\[
 |I(D)|\le M_q(S\cup\{q\})\le b.
\]
Both cases establish the required inequality. The argument includes
$b=0$ and does not require a stored maximizing subcollection.
\end{proof}

\begin{suppprop}[Corrected form of Claim~7]
At every completed insertion stage, let $L_i$ be an inserted language,
let $m_i$ be its assigned score, and let $S_i$ be its current prefix,
including $i$. Then
\begin{equation}
\label{app:pareto-corrected-bound}
 \text{for every }D\subseteq S_i,\qquad
 i\in D\text{ and }I(D)\text{ finite}
 \quad\Longrightarrow\quad |I(D)|\le m_i.
\end{equation}
\end{suppprop}
\begin{proof}
We induct on the number of insertions. The assertion is vacuous
before any language is inserted. Suppose it holds for the old
order and insert $L_q$.
For the new language, its assigned score is $m_q=M_q(S_q)$, so the
bound follows directly from the definition of the maximum in its
final prefix $S_q$. In particular, it is vacuous when there is no
candidate and remains valid when the maximum is zero.

Consider an old language $L_i$. If $q$ does not move to the left of
$i$, its prefix and assigned score remain unchanged, and the
inductive hypothesis applies. If $q$ moves to the left of $i$, let
$S_i$ be its old prefix. The relative order of all old languages
is unchanged, so its new prefix is exactly $S_i\cup\{q\}$.
When $q$ crosses $i$, the procedure's comparison gives
\[
 M_q(S_i\cup\{q\})\le m_i.
\]
The inductive hypothesis bounds every $i$-candidate in $S_i$ by
$m_i$. Applying Lemma~\ref{lem:pareto-score-invariant} with
$b=m_i$ proves the bound in its new prefix. Subsequent movements
of $q$ farther to the left change neither this set of indices nor
$m_i$. This covers every old language and completes the induction.
\end{proof}

\subsubsection{Recovery of the generation-time guarantee}

For completeness, we give the step connecting the corrected bound
to the original generation algorithm. Let $H\subseteq L_i$ be the
finite set of distinct observations from the target $L_i$, and suppose
that $i$ is included in the finite list currently being scanned and
that $|H|\ge m_i+1$. Let $B$ be the indices retained before the
scan reaches $i$. Every retained language contains $H$, and every
index in $B$ precedes $i$. Consequently
\[
 B\cup\{i\}\subseteq S_i,
 \qquad
 H\subseteq I(B\cup\{i\}).
\]
The target satisfies the sample-consistency test because
$H\subseteq L_i$. If the scan rejected it, its intersection with
the previously retained languages would be finite. Then
$B\cup\{i\}$ would be an $i$-candidate, and
\eqref{app:pareto-corrected-bound} would imply
\[
 m_i+1\le |H|
       \le |I(B\cup\{i\})|
       \le m_i,
\]
a contradiction. Thus the scan retains $L_i$.

The scan never removes a retained language and only adds a language
when the common intersection remains infinite. Its final
intersection is therefore infinite and, because it includes $L_i$
among the intersected languages, is a subset of $L_i$. Removing
the finite observed set leaves a point in $L_i\setminus H$.
The original generator's choice of such a point is consequently a
fresh target output.

In Theorem~8, the generator considers the first $f(r)$ languages
after $r$ observations, where $f$ is nondecreasing and unbounded.
Observations may repeat; let $H_r$ be the set of distinct points
observed by round $r$. Write $g(i)=\min\{r:f(r)\ge i\}$ for
the first round at which $L_i$ is included. If
\[
 |H_r|\ge t_i:=\max\{g(i),m_i+1\},
\]
then $r\ge |H_r|\ge g(i)$, so $f(r)\ge i$, and $|H_r|>m_i$.
The target is therefore under consideration and the preceding
argument guarantees a fresh target output. This proves the same
bound in terms of distinct observations as in the source paper,
including presentations with repetitions.
The score assignments, insertion comparisons, greedy
scan, and output rule all remain unchanged. Only the intermediate
proof assertion is replaced.

\paragraph{Formal verification.}
The two-language counterexample, the corrected score bound, and the
repaired generation guarantee, including presentations with repeated
observations, have been formalized and verified in Lean.
Effective oracle access and running-time complexity are outside this
verification scope.

%% file: supporting/appendix_research_workflow.tex
\subsection{Research design and workflow}
\label{app:research-method}
We supplied advanced reasoning models, including GPT-6 Astra with Ultra
reasoning and GPT-6 Pro, with the Lean library, paper maps, and relevant source
manuscripts. The models compared definitions and assumptions across papers,
looked for connections and simpler arguments, and proposed new questions.
Promising directions were developed through successive rounds of statement
refinement, proof construction, example testing, and review. The number of
iterations and the supplied material varied with the question.

The workflow combined mathematical exploration with focused checks. Follow-up
sessions examined proofs, searched for counterexamples, ran finite checks, and
formalized selected statements in Lean. For each featured finding,
the following proof subsections identify its source material, additional
mathematical contribution, written proof, and Lean verification status.

%% file: supporting/appendix_replay.tex
\subsection{Proper generation under replay}
\label{app:replay}
\newcommand{\replayset}[1]{\{#1\}}

This subsection proves the three-language impossibility discussed in
Section~\ref{sec:connections} and the finite-family characterization in
Theorem~\ref{thm:proper-replay}. We first specify the replay model, then
explain why the three-language construction prevents successful generation,
and finally identify exactly when a finite family admits a successful
learner. The model follows~\citet[][P22]{replay}; the three-language family
comes from~\citet[][P31]{memory}. The three-language impossibility, the
general characterization, and its consequences below have all been
formalized and verified in Lean.

\subsubsection{The replay model}

Let $\mathcal H=(L_1,\ldots,L_N)$, with $N\ge1$, be a finite family of
infinite subsets of a countable universe $U$, and let
$X=\bigcup_{i=1}^N L_i$. Different indices may name the same language.
The target is one fixed member $L_i$ of the family. At round $n\ge0$,
the learner first receives a point $x_n\in X$ and then outputs an index
$j_n\in[N]$. Thus a proper learner always proposes one of the languages
in $\mathcal H$, although that language need not contain every point
observed so far.

A \emph{complete replay text} for target $L_i$ satisfies
\begin{align}
 x_n&\in L_i\cup\bigcup_{t<n}L_{j_t}
       \qquad(n\ge0),\label{app:replay:eq:legal}\\
 L_i&\subseteq\replayset{x_n:n\ge0}.\label{app:replay:eq:full}
\end{align}
The first condition allows each observation to come either from the target
or from any language previously proposed by the learner. The second
condition requires every target point eventually to appear. Observations
may repeat, and there is no bound on how long the adversary may delay a
particular target point. Because there is no earlier output at round zero,
the first observation satisfies $x_0\in L_i$.

The learner succeeds if $L_{j_n}\subseteq L_i$ for all sufficiently large
$n$. It need not identify the target's index or output the whole target
language. We initially allow any deterministic learner that uses its full
input history. No computability restriction is imposed on the family or
on the learner. The finite-state construction below will use only the
membership profile of each observation, namely the set of languages that
contain it.

\subsubsection{The three-language impossibility}
\label{app:replay:construction}

The construction from the bounded-memory paper gives a useful first
example of what replay can hide. After the first output, the same complete
input stream can be legal for two targets that have no common allowable
output. The argument applies to learners with full history.

\begin{suppprop}[Three-language impossibility]
\label{app:replay:prop:triangle}
Let $T\subseteq U$ be infinite, and let $a_1,a_2,a_3\in U\setminus T$
be distinct. The family $L_i=T\cup\replayset{a_s:s\ne i}$, for
$i\in\{1,2,3\}$, does not admit deterministic proper generation in the
limit with replay.
\end{suppprop}

\begin{proof}
Suppose that a learner succeeds for this family. Give it a first point
$x_0\in T$, and let its first output be $j$. Write $p$ and $q$ for the
two indices different from $j$. Continue the input by enumerating every
point of $S=T\cup\replayset{a_1,a_2,a_3}$.

This is a complete legal replay text for either fixed target $L_p$ or
fixed target $L_q$. Indeed, the first point belongs to both targets,
and $S=L_p\cup L_j=L_q\cup L_j$. Every subsequent point is therefore
either a target point or a point of the first output language $L_j$,
which is already eligible for replay. The continuation also contains
every point of each target. Legality uses only the first output; it
places no restriction on the learner's later choices.

The learner sees the same input history in the two cases, so it produces
the same output sequence. If it succeeded for both targets, then after
the larger of the two success times every output language would be
contained in $L_p\cap L_q=T\cup\replayset{a_j}$. But every language in
the family contains two of the three exceptional points, so none is
contained in this intersection. This is a contradiction.
\end{proof}

\subsubsection{The finite-family characterization}

The preceding proof identifies the relevant ambiguity: two targets can
have the same union with the learner's first output language. Equivalently,
they agree outside that language. We now formulate a condition requiring
all targets that remain indistinguishable in this way to admit a common
allowable output.

For $x\in X$, define its membership profile by
$\operatorname{prof}(x)=\replayset{i\in[N]:x\in L_i}$, and let
$\mathcal P=\replayset{\operatorname{prof}(x):x\in X}$.
Every $P\in\mathcal P$ is nonempty. If the first observation has profile
$P$, precisely the indices in $P$ are possible targets at that point.
For a proposed first output $j$, write
$R_j(i)=L_i\setminus L_j$ for the part of target $L_i$ outside that
output. For each $i\in P$, define
\[
 E_j(i;P)=\replayset{k\in P:R_j(k)=R_j(i)}.
\]
This is the group of possible targets having the same residual as $L_i$.
The groups partition $P$. A common allowable output for such a group
is an index $h\in[N]$ satisfying $L_h\subseteq L_k$ for every
$k\in E_j(i;P)$. The language $L_h$ need not equal their intersection,
and its index need not belong to the group.

\begin{suppthm}[Exact finite characterization]
\label{app:replay:thm:main}
The family $\mathcal H$ admits deterministic proper generation in the
limit with replay if and only if the following condition holds:
for every $P\in\mathcal P$, there is an index $j\in[N]$ such that,
for every $i\in P$, there is an index $h\in[N]$ with
\[
 L_h\subseteq\bigcap_{k\in E_j(i;P)}L_k.
\]
The index $j$ need not belong to $P$.
\end{suppthm}

\begin{proof}
\emph{Necessity.}
Fix a realized profile $P$ and choose a first point $x$ with this profile.
Let $j$ be the first output of a learner that succeeds for every target
and every complete legal replay text. Consider any group
$E=E_j(i;P)$, and let $R$ denote the residual shared by its members.
After $x$, supply an enumeration of $S=L_j\cup R$.

For each $k\in E$, we have $x\in L_k$ and
$S=L_j\cup L_k$. Consequently the continuation is legal for target
$L_k$, because the first output permits replay from $L_j$, and it is
complete because it enumerates all of $L_k$. These are different
possible fixed targets for one and the same input sequence. The
learner's outputs are therefore identical in all cases.

By success on target $L_k$, all sufficiently late outputs must be
contained in $L_k$. There are only finitely many indices in $E$, so
we can take the largest of their success times. Every output after
that time is a member of $\mathcal H$ contained in every $L_k$ with
$k\in E$. At least one such output exists, giving the required index
$h$. Since this argument applies to every group, the first output
$j$ satisfies the condition for $P$.

\emph{Sufficiency.}
Assume the stated condition. For each realized profile $P$, fix an
index $j(P)$ satisfying it. Also, for every nonempty set of indices
$E\subseteq[N]$ for which a common allowable output exists, fix one
such index $h(E)$. These choices are made once for the whole family.
In particular, $h(E)$ depends only on $E$, not on the history by which
the learner reaches that set.

On receiving its first observation, the learner sets
$P=\operatorname{prof}(x_0)$, outputs $j=j(P)$, and initializes the
candidate set $C=P$. It keeps the index $j$ for the rest of the run.
For every subsequent observation $x$, it proceeds as follows:
\begin{enumerate}
\item If $x\notin L_j$, replace $C$ by
      $\replayset{k\in C:x\in L_k}$. If $x\in L_j$, leave $C$
      unchanged.
\item Examine the residuals $R_j(k)$ for $k\in C$. If one of them
      is contained in every other surviving residual, let $R$ be
      this least residual and let
      $E=\replayset{k\in C:R_j(k)=R}$. Output $h(E)$.
\item If there is no least residual, output $j$.
\end{enumerate}
On an empty candidate set the learner may output $j$; we will show
that this case never occurs on a legal replay text.

The distinction between a \emph{least} residual and an
\emph{inclusion-minimal} residual is essential here. A least residual
is contained in every surviving residual. An inclusion-minimal one
merely has no strictly smaller survivor; other survivors may be
incomparable with it. The proof needs containment in the true
residual, so the rule uses a least residual and otherwise returns
to $j$.

We first justify that $h(E)$ in the rule is defined. Initially $C=P$.
An update tests membership only at a point outside $L_j$. Two indices
with equal residuals give the same answer to every such test. Thus
an update either retains an entire original group $E_j(i;P)$ or
removes the entire group. By induction, $C$ is always a union of
these original groups. If a least surviving residual exists, its
indices form one entire original group, for which the assumed
condition provides a common allowable output.

We next prove the property that makes the updates reliable. Fix the
true target $i_*$. At every stage, the candidate set contains $i_*$,
and every output language lies inside $L_j\cup L_{i_*}$. Both
assertions hold initially: $x_0\in L_{i_*}$ gives $i_*\in P$, and
the first output is $j$. Suppose they hold for all earlier outputs,
and consider the next observation $x$. By replay legality, $x$
belongs either to $L_{i_*}$ or to an earlier output language. Every
such earlier language lies in $L_j\cup L_{i_*}$ by the induction
hypothesis. Therefore, if $x\notin L_j$, necessarily
$x\in L_{i_*}$. The update retains the true target, so $C$ stays
nonempty.

It remains to check the new output. Outputting $j$ preserves the
claimed containment immediately. Otherwise the learner outputs
$h(E)$ for a least residual $R$. Choose $k\in E$. Since
$L_{h(E)}\subseteq L_k$ and the true target is still in $C$, we have
\[
 L_{h(E)}\setminus L_j
 \subseteq L_k\setminus L_j
 =R
 \subseteq L_{i_*}\setminus L_j.
\]
Hence $L_{h(E)}\subseteq L_j\cup L_{i_*}$, as required. This
completes the induction. In particular, the construction satisfies
\begin{equation}
 L_{j_n}\setminus L_{i_*}\subseteq L_j
       \qquad(n\ge0).\label{app:replay:eq:contain}
\end{equation}
Thus any point outside the target that an output makes available for
replay was already available from the first output. The proof
establishes that later observations outside $L_j$ are trustworthy;
it does not assume this when defining the candidate updates.

Finally, we prove convergence. For every initial candidate $k\in P$
such that $R_j(i_*)\nsubseteq R_j(k)$, choose a point
$w_k\in R_j(i_*)\setminus R_j(k)$. This point belongs to the true
target and lies outside $L_j$, but does not belong to $L_k$.
Completeness guarantees that it eventually appears, at which point
$k$ is removed if it has not been removed already. There are only
finitely many initial candidates, so there is a time after all these
witnesses have appeared. From then on, every surviving residual
contains $R_j(i_*)$. Since $i_*$ survives, its residual is a least
surviving residual.

Moreover, every initial candidate with residual equal to
$R_j(i_*)$ survives forever: each observation used for an update is
a true-target point outside $L_j$, and hence belongs to every one
of these equal residuals. The group selected by the learner from
that time onward is therefore the fixed set $E_j(i_*;P)$. Its
chosen output $h(E_j(i_*;P))$ is contained in $L_{i_*}$, because
$i_*$ is a member of the group. The learner stabilizes to this
correct index, proving sufficiency.
\end{proof}

The proof also gives quantitative information about the learner and
shows that outputs can be restricted to inclusion-minimal languages
in the family. We state these consequences separately to distinguish
them from the existence criterion.

\begin{suppcor}[Finite memory and output choices]
\label{app:replay:cor:implementation}
Whenever the condition of Theorem~\ref{app:replay:thm:main} holds,
there is a learner that reads only membership profiles, uses at
most $1+N2^N$ states, changes its output index at most $N$ times,
and stabilizes to a correct index. It can be chosen so that every
output language is inclusion-minimal in $\mathcal H$, meaning that
no language in $\mathcal H$ is its proper subset. The containment
property~\eqref{app:replay:eq:contain} continues to hold.
\end{suppcor}

\begin{proof}
For the state bound, there is one initial state before the first
observation. Thereafter the learner need only store $(j,C)$, where
$j\in[N]$ and $C\subseteq[N]$. The update tests $j\in
\operatorname{prof}(x)$ and, when this fails, replaces $C$ by
$C\cap\operatorname{prof}(x)$. All residual comparisons and the
choices $j(P)$ and $h(E)$ are fixed data of the finite family.
Thus a finite-state transducer can implement the rule with at most
$1+N2^N$ states. This is an existence statement for the fixed
family, rather than an assertion that an arbitrary representation
allows these data to be computed.

After the first output, the output is a fixed function of $(j,C)$.
Consequently it can change only when $C$ strictly decreases,
apart from a possible change on the first subsequent round when
the initial output $j$ is replaced by the choice for the same
candidate set. Because $C$ starts nonempty and always contains the
true target, it strictly decreases at most $|P|-1\le N-1$ times.
There are therefore at most $N$ output changes. Stabilization and
containment were proved in the theorem.

To obtain inclusion-minimal outputs, first consider any index $j$
that satisfies the theorem's condition for a profile $P$. Since
the family is finite, there is an inclusion-minimal member $L_{j'}$
of the family with $L_{j'}\subseteq L_j$. If two targets agree
outside $L_{j'}$, they also agree outside $L_j$, because
$U\setminus L_j\subseteq U\setminus L_{j'}$. Each group formed
using $j'$ is therefore contained in a group formed using $j$.
A language contained in every target of the larger group is also
contained in every target of the smaller one. Thus $j'$ satisfies
the same condition, and every first-output choice can be made
inclusion-minimal.

Likewise, whenever a group $E$ admits a common allowable language,
choose an inclusion-minimal member of the family below that
language. It is still contained in every target in $E$. Fix these
minimal choices globally as the selectors $h(E)$ and run the same
construction. Every output is now inclusion-minimal, while the
state bound, output-change bound, convergence proof, and containment
argument remain valid.
\end{proof}

%% file: supporting/appendix_p29_frontier.tex
\subsection{The exact deterministic mistake--deadline frontier on a staircase}
\label{app:p29-frontier}

In this subsection we prove Theorem~\ref{thm:staircase-frontier}, which
determines the optimal mistake and convergence guarantees for the staircase
family. This family generalizes the construction in Theorem~6.4
of~\citet[][P29]{mistake} by allowing arbitrary positive finite block sizes.
We compare generators on every possible target simultaneously: improving
one target's guarantee may worsen another's. The proof constructs generators
from a sequence of binary choices, derives matching lower bounds for
arbitrary generators, and shows that every resulting choice expresses an
optimal tradeoff.

\subsubsection{Languages, histories, and guarantees}
\label{app:p29-frontier-setup}
We first give a concrete realization of the family in
Section~\ref{sec:new-mathematics}. Fix positive integers $a_1,a_2,\ldots$. On
$X=\mathbb N_{>0}\times\mathbb N_{>0}$, put
\[
 C_i=\{(i,j):1\le j\le a_i\},\qquad
 P_i=\bigcup_{k=1}^{i}C_k,\qquad
 s_i=|P_i|=\sum_{k=1}^{i}a_k,
 \qquad s_0=0.
\]
Write $R_i=\{(i,j):j>a_i\}$ for the private remainder of column $i$,
and define
\[
 L_i=P_i\cup R_i
     =P_{i-1}\cup(\{i\}\times\mathbb N_{>0}).
\]
Each $L_i$ is infinite. For $k<i$, its intersection with $L_k$ is
$P_k$. A point of $C_k$ belongs to exactly the languages $L_i$ with
$i\ge k$, whereas a point of $R_k$ belongs only to $L_k$.

A legal history for target $L_i$ is a finite ordered list of distinct
points in $L_i$. After a history of length $t\ge0$, a generator $G$
outputs a point outside its observed set, before the next observation
arrives. A generator is a deterministic function of the entire ordered
history; no computability restriction is imposed. An output is
a mistake when it lies outside the fixed target. This definition imposes
freshness relative to observed inputs, without requiring the outputs to
be mutually distinct.

Let $M_i(G)\in\mathbb N\cup\{\infty\}$ be the supremum of the total
number of mistakes over injective positive streams in $L_i$. Let
$D_i(G)\in\mathbb N\cup\{\infty\}$ be the least nonnegative integer
$D$ such that $G$ is valid on every legal $L_i$-history of every length
$t\ge D$; set $D_i(G)=\infty$ if no such finite $D$ exists. In
particular, an error after exactly $s$ observations forces
$D_i(G)\ge s+1$, and $D_i(G)=0$ means that no legal history receives
an invalid output. Time is the number of distinct observations,
including time zero at the empty history.

The same worst-case quantities result if streams are required to
enumerate the entire target. Every finite injective positive history
extends to such an enumeration. Thus every finite collection of
mistakes witnessed on a positive stream is also witnessed on some
complete text, and every history used to test a deadline occurs on a
complete text.

The two quantities answer different questions: $M_i(G)$ bounds how many
errors can occur, while $D_i(G)$ bounds how late an error can occur.
We compare both quantities for all targets at once. Specifically, $G'$ dominates $G$ if
$M_i(G')\le M_i(G)$ and $D_i(G')\le D_i(G)$ for every $i$. A profile
is Pareto minimal if no generator dominates it with a strict inequality
in at least one coordinate. Thus a Pareto-minimal profile cannot be improved
on any target without worsening at least one of the guarantees.

\subsubsection{Schedule characterization and proof}
\label{app:p29-frontier-proof}
The relevant decisions occur when the observed set is exactly $P_k$.
At that point the possible targets are $L_k,L_{k+1},\ldots$, whose
intersection is $P_k$, so no unobserved point is valid for all of them.
Choosing a point in $R_k$ is correct only for $L_k$; choosing a point
in $C_{k+1}$ is correct for every later target. We record these choices
by a binary sequence $\sigma=(\sigma_k)_{k\ge1}$, with $\sigma_k=1$
for the first choice and $\sigma_k=0$ for the second. The construction
below makes no mistakes away from these decision points.

For target $L_i$, an earlier choice of $R_k$, with $k<i$, causes one
mistake if the observed set reaches $P_k$. The choice of $C_{i+1}$
causes a mistake at $P_i$. These observations motivate the candidate
mistake bound and the corresponding set of possible error times:
\[
 m_i(\sigma)=\sum_{k<i}\sigma_k+1-\sigma_i,
 \qquad
 E_i(\sigma)=\{s_k:k<i,\ \sigma_k=1\}
              \cup\{s_i:\sigma_i=0\}.
\]
The proposed deadline is one more than the latest of these times, or
zero if there are none:
\begin{equation}
\label{app:p29-frontier-deadline}
 d_i(\sigma)=
 \begin{cases}
  0,&E_i(\sigma)=\varnothing,\\
  1+\max E_i(\sigma),&E_i(\sigma)\ne\varnothing.
 \end{cases}
\end{equation}

The first part of the theorem shows that these formulas are attained.
The second proves that every generator has guarantees at least as large
as those of one such sequence, simultaneously for every target. The third
establishes that none of the resulting profiles can be improved further.

\begin{suppthm}[Exact deterministic schedule frontier]
\label{app:p29-frontier-theorem}
For the staircase family above, the following assertions hold.
\begin{enumerate}
\item For every binary schedule $\sigma$, there is a deterministic
generator $G_\sigma$ satisfying, simultaneously for all $i\ge1$,
\begin{equation}
\label{app:p29-frontier-exact}
 M_i(G_\sigma)=m_i(\sigma),\qquad
 D_i(G_\sigma)=d_i(\sigma).
\end{equation}
\item Every deterministic generator $G$ is dominated by a schedule:
there is one $\sigma$ such that, for all $i\ge1$,
\begin{equation}
\label{app:p29-frontier-domination}
 m_i(\sigma)\le M_i(G),\qquad d_i(\sigma)\le D_i(G).
\end{equation}
\item Every schedule profile is Pareto minimal. Consequently the
profiles in~\eqref{app:p29-frontier-exact} are exactly the deterministic
Pareto frontier.
\end{enumerate}
\end{suppthm}
\begin{proof}
\emph{Attaining the proposed guarantees.}
We first define $G_\sigma$ on every finite history and then prove the
two equalities in~\eqref{app:p29-frontier-exact}.
Choose fixed orders for making all the selections below. If a history
contains a point of $R_i$, any compatible target must be $L_i$.
At such a realizable history, output any unobserved point of $L_i$.
If a history is unrealizable, output any unobserved point of $X$;
this branch is irrelevant to targetwise guarantees but makes the
generator total and fresh.

It remains to define the generator before a private point appears.
At the empty history, output a point of $C_1$, which belongs to every
target. At a nonempty history consisting only of finite-block points,
let $k$ be the largest observed column index and let $S$ be its
observed set. The compatible targets are precisely
$L_k,L_{k+1},\ldots$, and their intersection is $P_k$. Indeed, the
observed point of $C_k$ excludes every earlier target, while every
target of index at least $k$ contains all the observed points. Their
common core is $P_k$ because $L_k\cap L_{k+1}=P_k$.

If $S\ne P_k$, output a point of $P_k\setminus S$. It is fresh and
valid for every compatible target. If $S=P_k$, no fresh point is valid
for all compatible targets. When $\sigma_k=1$, output a point of
$R_k$, valid only for $L_k$. When $\sigma_k=0$, output a point of
$C_{k+1}$, valid for every later target and invalid for $L_k$. Both
choices are fresh.

Fix a target $L_i$. Its stream can encounter a boundary $S=P_k$
only for $k\le i$, and it can encounter each such boundary at most
once because observations are injective. Every output away from
these boundaries is valid. At a boundary $k<i$, a mistake occurs
exactly when $\sigma_k=1$; at boundary $k=i$, a mistake occurs
exactly when $\sigma_i=0$. The respective observed counts are $s_k$.
This proves $M_i(G_\sigma)\le m_i(\sigma)$ and
$D_i(G_\sigma)\le d_i(\sigma)$.

To attain these bounds, present $C_1,C_2,\ldots,C_i$ in fixed
internal orders, followed by an enumeration of $R_i$. This is an
injective complete text of $L_i$, and it encounters every boundary
$P_k$ for $k\le i$. It therefore realizes exactly the mistakes in
$E_i(\sigma)$. If that set is nonempty, its latest mistake occurs
after $\max E_i(\sigma)$ observations, forcing deadline at least
$1+\max E_i(\sigma)$. If it is empty, the preceding upper bound
already gives deadline zero. This establishes both equalities
in~\eqref{app:p29-frontier-exact}.

\emph{Lower bounds for an arbitrary generator.}
We next show that allowing more elaborate choices, including dependence
on the order of observations, cannot improve on all the profiles above.
Fix a deterministic fresh generator $G$. For each $k$, let $h_k$
be the ordered concatenation of $C_1,\ldots,C_k$ using the same
fixed internal orders. Define
\[
 \sigma_k=1\quad\Longleftrightarrow\quad G(h_k)\in R_k.
\]
Freshness gives $G(h_k)\notin P_k$. Hence if $\sigma_k=0$,
the output at $h_k$ is outside $L_k=P_k\cup R_k$.

For a fixed target $L_i$, extend $h_i$ by the private remainder
$R_i$ to obtain one complete $L_i$-text. At every earlier boundary
$h_k$ with $\sigma_k=1$, the output belongs to $R_k$ and is
invalid for $L_i$. At its own boundary $h_i$, the output is invalid
if $\sigma_i=0$. Thus this one text witnesses all
$m_i(\sigma)$ charged mistakes, at exactly the observed counts
in $E_i(\sigma)$. The resulting mistake and deadline lower bounds
are~\eqref{app:p29-frontier-domination}; other errors made by $G$
can only increase those quantities.

The infinite concatenation of all finite blocks is used only to
specify a consistent family of finite histories. It is not asserted
to be a text of any one target. Each lower-bound argument above
fixes $i$ and uses the legal prefix $h_i$ followed by a legal
continuation for that same target.

\emph{Optimality of every schedule.}
The lower bound shows that the schedules suffice to find all optimal
guarantees. We must still prove that every schedule is itself optimal.
First, no two distinct schedules have coordinatewise ordered
mistake vectors. Suppose otherwise that
$m_i(\tau)\le m_i(\sigma)$ for every $i$, and let $k$ be the
first index where $\tau_k\ne\sigma_k$. The inequality at $i=k$
forces $\tau_k=1$ and $\sigma_k=0$. At $i=k+1$, the difference
of the mistake counts is
\[
 m_{k+1}(\tau)-m_{k+1}(\sigma)
 =1-(\tau_{k+1}-\sigma_{k+1}).
\]
For this to be nonpositive, we must also have
$\tau_{k+1}=1$ and $\sigma_{k+1}=0$. But then
\[
 m_{k+2}(\tau)-m_{k+2}(\sigma)
 =2-(\tau_{k+2}-\sigma_{k+2})\ge1,
\]
a contradiction. Therefore $\tau=\sigma$ whenever all the
mistake inequalities hold.

Now suppose an arbitrary $G$ dominates $G_\sigma$. The reduction
just proved supplies a schedule $\tau$ with
$m_i(\tau)\le M_i(G)\le m_i(\sigma)$ for every $i$, so
$\tau=\sigma$. Applying the deadline part of the same reduction
gives $d_i(\sigma)\le D_i(G)\le d_i(\sigma)$, while the mistake
inequalities also become equalities. No coordinate can improve
strictly. Thus every schedule profile is Pareto minimal. Conversely,
every Pareto-minimal generator must have the same profile as its
dominating schedule, proving the final assertion.
\end{proof}

The two constant schedules illustrate the tradeoff. Always choosing
$C_{k+1}$ makes no mistakes at the earlier decision points for target
$L_i$, but makes one at $P_i$. Its guarantees are therefore $M_i=1$
and $D_i=s_i+1$. Always choosing $R_k$ avoids the mistake at $P_i$
but makes one at each earlier decision point, giving $M_i=i-1$ and
$D_i=s_{i-1}+1$ for $i\ge2$, with $M_1=D_1=0$.
Theorem~\ref{app:p29-frontier-theorem} determines all intermediate
choices as well. Its lower bound applies to every deterministic generator,
and its minimality statement shows that each listed profile is an optimal
tradeoff across the targets, completing the special-case characterization.

%% file: supporting/appendix_autoformalization.tex
\section{Library-Assisted Lean Proof Generation: Additional Experiment Details and Results}
\label{app:autoformalization}

This appendix provides additional experiment details and results for the Lean proof generation experiments in Section \ref{sec:autoformalization}.

\subsection{Additional Formalization-Evaluation Details}
\label{app:autoformalization_more_details}

\subsubsection{Generation and selection of theorem tasks}

The five theorem tasks used in our experiments in Section \ref{sec:autoformalization} were generated from an
upstream ideation task over the same five-paper research neighborhood.
The ideation model was asked to propose six mathematically distinct theory contributions per run, each connecting the anchor paper P39 with at least one of the four companion papers P06, P12, P17, or P19.  
For every proposal, it had to state a precise candidate theorem or counterexample, identify the closest supplied
results from the papers or the research library, outline a proof or counterexample route, state the main risk, and stress-test at least one boundary case. 
It then self-selected at most three strongest finalists. 
We generated research proposals using Codex CLI 0.153.1 with the gpt-5.6-sol model, reasoning effort set to high, and output verbosity set to medium.
Nine ideation runs used Codex CLI 0.153.1 with \texttt{gpt-5.6-sol} reasoning effort high and produced 27 run-level proposals.

The finalist proposals were blind-reviewed by a separate \texttt{ChatGPT-5.6 Pro} reviewer. 
We rejected false and insufficiently supported proposals and proposals that were immediate consequences of results in the supplied papers or the formal library. 
Five surviving proposals were developed into the fixed tasks Case 002, Case 017, Case 024, Case 025, and Case 019 (corresponding to Tasks A, B, C, D, E in Section \ref{sec:autoformalization}, respectively).  
Thus, the tasks are newly generated and were not available to the models at the time of experiment evaluation.

\subsubsection{Natural-language proofs and fixed Lean targets}

Before a theorem was admitted to the LLM formalization evaluation, we prepared (i) a natural-language theorem statement, (ii) a complete natural-language proof, and (iii) a faithful Lean theorem statement.  
The Lean theorem statement used our Lean library vocabulary whenever its semantics matched the prose statement exactly. 
Finally, we obtained a controller-private complete Lean proof to guarantee that a successful Lean proof is attainable, and provided a task-specific Lean library subset for the experiment condition PML-Oracle.
The controller-private proof was never included in an LLM author-visible packet.
Table \ref{tab:task-resource-counts} lists the task-specific formal resources.

\begin{table}[t]
  \centering
  \small
  \caption{Task-specific formal resources. Vocabulary files are used for Lean theorem statements and are provided for all conditions.  PML-Oracle files are oracle-selected Lean files for the condition PML-Oracle, which form a subset of the 161-file frozen sub-library.}
  \label{tab:task-resource-counts}
  \begin{tabular}{@{}lcc@{}}
    \toprule
    Task & Vocabulary files & PML-Oracle files \\
    \midrule
    Case 002      & 6 & 4  \\
    Case 017 & 8 & 10 \\
    Case 024 & 6 & 13 \\
    Case 025 & 8 & 44 \\
    Case 019 & 8 & 59 \\
    \bottomrule
  \end{tabular}
\end{table}

\subsubsection{Formalization prompt and author-visible context}

The top-level LLM author prompt was identical across P, PML, and PML-Oracle, with theorem-specific namespaces and output filenames.  In the with-proof conditions, the prompt template was:

\begin{verbatim}
Read AGENTS.md, SOURCE_CONFIG.md, TASK_STAGE3.md,
THEOREM_STATEMENT.md, CANONICAL_FULL_PROOF.md, and
Stage3Model.lean. Consult any additional supplied source
material, TargetTemplate.lean, and LEAN_CHECK.md selectively.
Implement the exact stage3_result : <MainClaim>, adding local
helpers as needed. Check complete entry-point versions before
optional refactoring. Leave the final Lean source and a brief
STATUS.md. You have 90 minutes including tools and compilation.
Work autonomously and do not ask questions.
\end{verbatim}

For the no-proof conditions, the prompt was identical except that
\texttt{CANONICAL\_FULL\_PROOF.md} and the instruction to read it were removed.
Reasoning effort (high or medium) was set in the inference configuration and
was not described in the natural-language task prompt.

The referenced task file fixed the root declaration to the following schema:

\begin{lstlisting}[style=lean]
import Stage3Model

theorem stage3_result : <MainClaim> := by
  sorry
\end{lstlisting}

The LLM author had to replace the \texttt{sorry} placeholder without changing the shared model, weakening the conclusion, or adding assumptions.  
It could create local helper files in \texttt{output/} and repeatedly call the supplied Lean checker.
The task requested a final Lean entry point and a short diagnostic status file; the status text was not used to determine success.  
LLM authors were instructed to save the strongest checked partial development if they could not finish and
were prohibited from using \texttt{sorry}, \texttt{admit}, new axioms, unsafe code, kernel bypasses, external files, prior runs, network access, or other agents.

All conditions received the same five papers, theorem statement, proof-availability setting, exact Lean target, Mathlib environment, task text, checker policy, and 90-minute author budget.

All experiments were run on CPU nodes of the Delta high-performance computing cluster.
Model inference was run through OpenRouter.

\subsection{Additional Experiment Results}
\label{app:stage3-additional-results}

In this section, we report additional experiment results. We further disaggregate success by theorem task and report additional diagnostics on API cost and resource usage. 

\subsubsection{Full task-level success}

Table~\ref{tab:stage3-full-task-success} gives the complete success counts across theorem tasks, resource conditions, proof availability, and reasoning effort.
Across all conditions, PML produces 79 valid proofs out of 100 trajectories, compared with 48/100 for P.
This advantage appears across multiple theorem tasks and persists both with and without a supplied natural-language proof.

The clearest example is Case 025, where P fails in all 20 trajectories while PML succeeds in all 20.
Table~\ref{tab:stage3-results}(b) suggests why: Case 025 is closely aligned with existing library machinery, and successful PML proofs reuse a large fraction of proof-relevant library declarations, leaving mainly task-specific interface lemmas to be established.
In contrast, P must reconstruct the underlying arguments from the minimal project vocabulary.
A similar pattern appears for Case 019: P again has no successful trajectory, whereas PML succeeds in 5/20 runs and PML-Oracle in 9/20.
PML-Oracle, which provides an idealized retrieval ceiling, achieves slightly higher success overall and a larger gain on Case~019.

\begin{table}[t]
  \centering
  \small
  \setlength{\tabcolsep}{8pt}
  \caption{Full task-level success.
  Each entry is the number of successful trajectories out of five independent runs. The tasks Case 002, 017, 024, 025, 019 correspond to Tasks A, B, C, D, E in Section \ref{sec:autoformalization}, respectively.}
  \label{tab:stage3-full-task-success}
  \begin{tabular}{@{}lccc@{}}
    \toprule
    Setting & P & PML & PML-Oracle \\
    \midrule

    \multicolumn{4}{@{}l}{\textbf{Case 002}} \\
    Proof / High      & 5 & 4 & 5 \\
    No proof / High   & 4 & 5 & 5 \\
    Proof / Medium    & 2 & 2 & 2 \\
    No proof / Medium & 2 & 4 & 3 \\

    \addlinespace
    \multicolumn{4}{@{}l}{\textbf{Case 017}} \\
    Proof / High      & 5 & 5 & 5 \\
    No proof / High   & 5 & 5 & 5 \\
    Proof / Medium    & 4 & 5 & 4 \\
    No proof / Medium & 5 & 4 & 5 \\

    \addlinespace
    \multicolumn{4}{@{}l}{\textbf{Case 024}} \\
    Proof / High      & 5 & 5 & 5 \\
    No proof / High   & 5 & 5 & 5 \\
    Proof / Medium    & 4 & 5 & 5 \\
    No proof / Medium & 2 & 5 & 3 \\

    \addlinespace
    \multicolumn{4}{@{}l}{\textbf{Case 025}} \\
    Proof / High      & 0 & 5 & 5 \\
    No proof / High   & 0 & 5 & 5 \\
    Proof / Medium    & 0 & 5 & 5 \\
    No proof / Medium & 0 & 5 & 5 \\

    \addlinespace
    \multicolumn{4}{@{}l}{\textbf{Case 019}} \\
    Proof / High      & 0 & 3 & 5 \\
    No proof / High   & 0 & 2 & 4 \\
    Proof / Medium    & 0 & 0 & 0 \\
    No proof / Medium & 0 & 0 & 0 \\

    \bottomrule
  \end{tabular}
\end{table}

\subsubsection{API cost}
For any collection of trajectories, we define failure-adjusted API cost as
\[
  \frac{\text{total API cost over the trajectories}}
       {\text{number of kernel-validated proofs}}.
\]
Failed trajectories are included in the numerator. The quantity is undefined when no trajectory in the group succeeds.
Table~\ref{tab:stage3-cost-breakdown} reports this measure by theorem task and by proof-availability/reasoning setting.

Across the full experiment, PML reduces the failure-adjusted API cost per valid proof from \$2.76 for P to \$1.95, with PML-Oracle reducing it further to \$1.78. This pattern is consistent across all factorial settings: PML has lower cost than P for every combination of proof availability and reasoning effort. The task-level results show the same general pattern whenever both conditions produce valid proofs.

\begin{table}[t]
\centering
\small
\caption{Failure-adjusted API cost per valid proof.
Panel A pools over proof availability and reasoning effort within each theorem task.
Panel B pools over theorem tasks within each experimental setting. The tasks Case 002, 017, 024, 025, 019 correspond to Tasks A, B, C, D, E in Section \ref{sec:autoformalization}, respectively.}
\label{tab:stage3-cost-breakdown}
\begin{tabular}{@{}lccc@{}}
\toprule
 & P & PML & PML-Oracle \\
\midrule
\multicolumn{4}{@{}l}{\textbf{Panel A: By theorem task}} \\
Case 002      & \$2.28 & \$2.06 & \$1.91 \\
Case 017  & \$1.56 & \$1.13 & \$0.99 \\
Case 024  & \$2.05 & \$0.89 & \$0.95 \\
Case 025  & --     & \$1.43 & \$1.34 \\
Case 019  & --     & \$11.10 & \$5.90 \\
\addlinespace
\multicolumn{4}{@{}l}{\textbf{Panel B: By proof availability and reasoning effort}} \\
Proof / High      & \$2.62 & \$2.05 & \$1.84 \\
No proof / High   & \$2.76 & \$1.98 & \$1.68 \\
Proof / Medium    & \$2.98 & \$1.89 & \$1.94 \\
No proof / Medium & \$2.77 & \$1.87 & \$1.70 \\
\midrule
Overall           & \$2.76 & \$1.95 & \$1.78 \\
\bottomrule
\end{tabular}
\end{table}

\subsubsection{Resource usage}
Table~\ref{tab:stage3-resource-per-run} reports resource consumption, including tokens and wall-clock time. Each quantity
is normalized by the number of scheduled 100 trajectories per condition, including both successful and failed runs.

All three resource conditions process a similar amount of input tokens per trajectory.
The resource files were available on disk and entered model
context only when the LLM agent chose to inspect them, while the input-token measure accumulates context across many API calls and includes cached prefixes. 
PML has more resource files than P (and thus a higher cache-write count), but uses fewer API calls.
PML uses fewer reasoning tokens and less wall-clock time.

\begin{table*}[t]
  \centering
  \footnotesize
  \setlength{\tabcolsep}{4pt}
  \caption{Aggregated resource usage per scheduled trajectory.  Each resource total
  is divided by the 100 total scheduled trajectories in the corresponding arm. Input tokens are cumulative across API calls and include cached prefixes; cache-write tokens measure newly cached
  prompt prefixes.  Author time excludes scheduler queueing, setup, final validation, and packaging.}
  \label{tab:stage3-resource-per-run}
  \begin{tabular*}{\textwidth}{@{\extracolsep{\fill}}lcccccc@{}}
    \toprule
    Condition &
    \shortstack[c]{Valid\\proofs} &
    \shortstack[c]{API calls\\per run} &
    \shortstack[c]{Input tokens\\per run} &
    \shortstack[c]{Cache-write\\tokens per run} &
    \shortstack[c]{Reasoning\\tokens per run} &
    \shortstack[c]{Author time\\per run} \\
    \midrule
    P              & 48/100 & 55.6 & 3.122M & 166.6K & 11.3K & 19.3 min \\
    PML            & 79/100 & 51.9 & 3.121M & 253.8K & 9.33K & 15.7 min \\
    PML-Oracle & 81/100 & 53.6 & 3.184M & 210.2K & 9.94K & 17.6 min \\
    \bottomrule
  \end{tabular*}
\end{table*}

%% file: supporting/appendix_experiments.tex
\section{Mathematical-Reading Experiments}
\label{app:experimental-documentation}
\label{app:experiment-details}
This appendix gives the experimental setup and accuracy calculation for
the mathematical-reading experiments in Section~\ref{sec:library-value}.

\label{app:reading-runtime}
Advanced models, including Claude Opus 5 Max and GPT-6 Astra Xhigh,
drafted five-option questions about the source papers and library. A separate
audit using GPT-6 Astra Ultra through Codex checked the answers and supporting
sources. Of 1,401 candidate questions,
1,193 were retained: 825 about 16 individual papers and 368 from a five-paper
collection. The audit excluded 206 questions judged incorrect and two judged unverifiable.

The individual-paper bank covers P01 \citep{km2024}, P02 \citep{li2024generation}, P03 \citep{kalavasis2025limits}, P04 \citep{charikar2024facets}, P05 \citep{kalavasis2025characterizations}, P06 \citep{ananth2025generation},
P08 \citep{karbasi2025detection}, P09 \citep{prr2025}, P10 \citep{hanneke2025union}, P12 \citep{bai2026noise}, P17 \citep{mehrotra2026language}, P19 \citep{li2026characterizing}, P23 \citep{kleinberg2026banach}, P28 \citep{li2026contrastive}, P31 \citep{memory}, and P39 \citep{cai2026dense}.
For the map study, we use the five-paper subset of the above 16 papers on generation and noise: P02, P06,
P12, P17, and P19.

For individual papers, we compared the question alone, paper excerpts, and
excerpts supplemented with Lean definitions and theorem statements. Relevant
Lean was selected using the question's recorded source declaration, followed
by lexical matching; Lean in file order used the original source ordering.
Unrelated Lean supplied a control with similar formatting and character length.
The added Lean targeted 6,000 characters. For the five-paper collection, all
conditions supplied the same excerpts, with no map, the relevant map, or a sham
map. The relevant map summarized theorems and relationships obtained from
library imports and dependencies. The sham contained Lean statements from
outside the five-paper collection, with the same heading and character length
(19,102 characters).

The reader was Qwen3.6-27B, run in BF16 with vLLM on one H100 80GB GPU.
Thinking was disabled and temperature was zero. Each request supplied a fixed
text packet; the reader selected an option using the first-token
log-probabilities of A--E, without tools or a Lean compiler.

We first screened questions using two cyclic answer-order rotations without
supporting material, retaining questions with accuracy at most one half.
Evaluation used three different rotations, with the same schedule for every
information condition on a given question. This left 404 individual-paper
questions and 204 questions from the five-paper collection.

\paragraph{Accuracy calculation.}
\label{app:reading-analysis}
For each question, we average correctness over the three evaluation rotations.
To avoid giving repeated questions about the same source result extra weight,
we group questions with the same recorded paper identifiers and source label
(such as a theorem or definition number); a question without a source label
forms its own group. We average the question scores within each group and
then average equally across groups. Within each study, the same questions,
rotations, and groups are used in every information condition.